\documentclass[conference]{ESANN}

\ifCLASSINFOpdf
   \usepackage[pdftex]{graphicx}
\else
\fi
\usepackage{dsfont, amsmath, amssymb, amsthm}
\usepackage{amsmath}
\usepackage{amssymb}
\usepackage{multirow}
\usepackage{array}
\usepackage[lofdepth,lotdepth]{subfig}
\usepackage{times}
\usepackage{graphicx} 

\newtheorem{theorem}{Theorem}[section]
\newtheorem{lemma}[theorem]{Lemma}
\newtheorem{proposition}[theorem]{Proposition}

\begin{document}

\title{Energy Saving Additive Neural Network \ \\}
\author{\IEEEauthorblockN{{Arman Afrasiyabi$^{1 \star}$}, {Ozan Yildiz$^{1 \star}$}, {Baris Nasir$^1$}, {Fatos T. Yarman Vural$^1$} and {A. Enis Cetin$^2$}} \\
\IEEEauthorblockA{{$1$ Computer Engineering Department, Middle East Technical University, Ankara, Turkey} \\ Email: \{afrasiyabi.arman, oyildiz, bnasir, vural\}@ceng.metu.edu.tr \\ 
$\star$ Equal contribution \\ \ \\ 
{$2$ Electrical and Electronics Engineering, Bilkent University, Ankara, Turkey} \\ 
Email: cetin@bilkent.edu.tr} 
\footnote{This is not a lie.}
}

\maketitle

\begin{abstract}
In recent years, machine learning techniques based on neural networks for mobile computing become increasingly popular. Classical multi-layer neural networks require matrix multiplications at each stage. Multiplication operation is not an energy efficient operation and consequently it drains the battery of the mobile device. In  this paper, we propose a new energy efficient neural network with the universal approximation property over space of Lebesgue integrable functions. This network, called, additive neural network, is very suitable for mobile computing. The neural structure is based on a novel vector product definition, called ef-operator, that permits a multiplier-free implementation. In ef-operation, the "product" of two real numbers is defined as the sum of their absolute values, with the sign determined by the sign of the product of the numbers. This "product" is used to construct a vector product in $R^N$. The vector product induces the $l_1$ norm. The proposed additive neural network successfully solves the XOR problem. The experiments on MNIST dataset show that the classification performances of the proposed additive neural networks are very similar to the corresponding multi-layer perceptron and convolutional neural networks (LeNet).

\section{Introduction}

\end{abstract}

    
    
    Artificial Neural Networks (ANN) have been shown to solve many real world problems, such as, computer vision, natural language processing, recommendation systems and many other fields \cite{lecun2015deep}. Convolutional Neural Network (CNN) architectures achieve human performance in many computer vision problems including image classification tasks \cite{lecun1998gradient, krizhevsky2012imagenet, simonyan2014very, taigman2014deepface, szegedy2015going}. However, the number of parameters in these high-performance networks ranges from millions to billions which require computers capable of handling high computational complexity, high energy and memory size. Consequently, the minimal computational environment for such a network is a desktop computer with a powerful CPU and a dedicated high-end GPU.

	Recent developments in VLSI industry create powerful mobile devices which can be used in many practical recognition applications. ANNs are already being used in drones and unmanned aerial vehicles for flight control, path estimation \cite{calise1998nonlinear}, obstacle avoidance and human recognition like abilities \cite{giusti2016machine} (DJI Phantom 4).

	However, the current structure of the ANNs, especially, deep networks, prohibits us to implement these algorithms effectively on mobile devices due to high energy requirements. A typical neuron needs to perform three main tasks to produce an output: (i) an inner product operation involving multiplication of inputs by weights, (ii) addition, and (iii) pass the result of the inner product through an activation function. According to the \cite{han2015learning}, the multiplication operation is the most energy consuming operation.  

    In this paper, we propose an $l1$ norm based energy efficient neural network, called \textit{additive neural network}, that replaces the multiplication operation with a new energy efficient operator, called \textit{ef-operator}. Instead of multiplications, we use sign multiplications and addition operations in a typical neuron. The sign multiplication of two real numbers is a simple bit operation. An addition consumes relatively lower energy compared to a regular multiplication as shown in \cite{han2015learning} in most processors. Our object recognition experiments on MNIST and CIFAR datasets show that we are able to match the performance of the state of the art neural networks
    without performing any other changes on the ANN structure.
    
    In Section 2, we review the related work in energy efficient neural network design. In Section 3, we define a new vector product and the corresponding operator, called ef-operator. In Section 4, we introduce the additive neural network, based on the ef-operator. In Section 5, we made a brief analysis for the existence and convergence problems of the proposed additive neural network. Section 6, provides the experimental results to compare the performance of the proposed additive neural network with multi-layer perceptron and convolutional neural networks. Finally, Section 7 concludes the paper. 

\section{Related Work}

	Due to large size of the parameter space, artificial neural networks are generally computationally prohibitive and become inefficient in terms of energy consumption and memory allocation. Several approaches from different perspectives have been proposed to design computationally efficient neural network structures to handle high computational complexity. 
   
   We first introduced the $l1$ norm based vector product for some image processing applications in 2009 \cite{tuna2009image,suhre2013multiplication,akbacsenergy,demir2016co}. We also proposed the multiplication free neural network structure in 2015 \cite{akbacs2015multiplication}. However, the recognition rate was below $10\%$of a regular neural network. In this article, we are able to match the performance of regular neural networks by introducing a scaling factor to the $l1$ norm based vector product and new training methods. We are only $0.034\%$ below the recognition rate of a regular neural network in MNIST dataset.  
    
	Other solutions to energy efficient neural networks include dedicated software for a specific hardware, i.e. neuromorphic devices \cite{esser2016convolutional,painkras2013spinnaker,pfeil2012six,moradi2014event,park201465k}. Although such approaches reduces energy consumption and memory usage, they require special hardware. Our neural network framework can be implemented in ordinary microprocessors and digital signal processors.
    
	Sarwar et al. used the error resiliency property of neural networks and proposed an approximation to multiplication operation on artificial neurons for energy-efficient neural computing \cite{sarwar2016multiplier}. They approximate the multiplication operation by using the Alphabet Set Multiplier (ASM) and Computation Sharing Multiplication (CSHM) methods. In ASM, the multiplication steps are replaced by shift and add operators which are performed by some alphabet defined by a pre-computer bank. This alphabet is basically a subset of the lower order multiplies of the input. The multiplies that are not exist in the computed subset are approximated by rounding them to nearest existing multiplies. This method reduces the energy consumption since addition and bit shifting operations are much efficient than the multiplication. Therefore, the smaller sized alphabets result in a more efficient architecture. Additionally, they define a special case called Multipler-less Artificial Neuron (MAN), in which there is only one alphabet for each layer. This method provides more energy efficiency with a minimum accuracy loss. It should be noted that this method is applied on test stages, therefore, the training step still uses the conventional method.
 
	Han et al. proposed a model that reduces both computational cost and storage by feature learning \cite{han2015learning}. Their approach consists of three steps. In the first step, they train the network to discriminate important features from redundant ones. Then, they remove the redundant weights, and occasionally neurons, according to a threshold value to obtain a sparser network. This step reduces the test step's cost. At the final step they retrain the network to fine tune the remaining weights. They state that this step is much more efficient than using the fixed network architecture. They tested the proposed network architecture with ImageNet and VGG-16. The parameter size for these networks reduces between $\times 9$ to $\times 13$ without any accuracy loss. 
    
    Abdelsalam et al. approximate the tangent activation function using the Discrete Cosine Transform Interpolation Filter (DCTIF) to run the neural networks on FPGA boards efficiently \cite{abdelsalam2016accurate}. They state that DCTIF approximation reduces the computational complexity at the activation function calculation step by performing simple arithmetic operations on stored samples of the hyperbolic tangent activation function and input set. The proposed DCTIF architecture divides the activation function into three regions, namely, pass, process and saturation regions. In the pass region the activation function is approximated by y = x and in the saturation region the activation function is taken as y = 1. The DCTIF takes place in the process region. Parameters of the transformation should be selected carefully to find a balance between computational complexity and accuracy. They have shown that the proposed method achieve significant decrease on energy consumption while keeping the accuracy difference within $1\%$ with conventional method.
    
    Rastegari et al. proposes two methods to provide efficiency on CNNs. The first method, Binary-Weight-Networks, approximates all the weight values to binary values \cite{rastegari2016xnor}. In this way the network needs less memory (nearly $\times 32$). Since the weight values are binary, convolutions can be estimated by only addition and subtraction, which eliminates the main power draining multiplication operation. Therefore, this method both provides energy efficiency and faster computations.
    
    The second method proposed by them is called XNOR-Networks where both weights and inputs to the convolutional and fully connected layers are approximated by binary values. This extends the earlier proposed method by replacing addition and subtraction operations with XNOR and bitcounting operations. This method offers $\times 58$ faster computation on CPU on average. While this method enables us to run CNNs on mobile devices, it costs $12\%$ loss accuracy on average.

\section{A New Energy Efficient Operator}
\label{sec:ANEEO}
	Let $\mathbf{x}$ and $\mathbf{y}$ be two vectors in $\mathds{R}^{d}$. We define an new operator, called ef-operator, as the vector product of $\mathbf{x}$ and $\mathbf{y}$ as follows;

	\begin{equation}
		\label{eq:operator_definition1}
		\mathbf{x}\diamond \mathbf{y} := \sum_{i=1}^d \text{sign}(x_i \times y_i ) (|x_i|+ |y_i|),
	\end{equation}

	which can also be represented as follows;

	\begin{equation}
		\label{eq:operator_definition}
        \mathbf{x}\diamond \mathbf{y} := \sum_{i=1}^d \text{sign}(x_i)y_i + \text{sign}(y_i)x_i,
	\end{equation}
  
	where $\mathbf{x}=[x_1,\,\hdots,\,x_d]^T,\mathbf{y}=[y_1,\,\hdots,\,y_d]^T\in\mathds{R}^d$. The new vector product operation does not require any multiplications. The operation $(x_i \times y_i ) (|x_i|+ |y_i|)$ uses the sign of the ordinary multiplication but it computes the sum of absolute values of $x_i$ and $y_i$. ef-operator, $\diamond$, can be implemented without any multiplications. It requires summation, unary minus operation and if statements which are all energy efficient operations.

	Ordinary inner product of two vectors induces the $\ell_2$ norm. Similarly, the new vector product induces a scaled version of the  $\ell_1$ norm:

	\begin{equation}
		\label{eq:l1h}
		\mathbf{x}\diamond \mathbf{x} = \sum_{i=1}^d \text{sign}(x_i \times x_i ) (|x_i|+ |x_i|) = 2|| x ||_1
	\end{equation}
	Therefore, the ef-operator performs a new vector product, called $\ell_1$ product of two vectors, defined in Eq.~\ref{eq:operator_definition1}.

	We use following notation for a compact representation of ef-operation of a vector by a matrix. Let $\mathbf{x}\in \mathds{R}^{d}$ and $\mathbf{W}\in \mathds{R}^{d \times M}$ be two matrices, then the ef-operation between $\textbf{W}$ and $x$ is defined as follows;

	\begin{equation}
		\label{eq:operator_definition_matrix}
		\mathbf{x}\diamond \mathbf{W} := \begin{bmatrix} \mathbf{x} \diamond \mathbf{w}_1 & \hdots & \mathbf{x} \diamond \mathbf{w}_M \end{bmatrix}^T \in \mathds{R}^M,
	\end{equation}

	where $\mathbf{w}_j$ is $j$th column of $\mathbf{W}$ for $j=1,\,2,\,\hdots,\,M$.

\section{Additive Neural Network with ef-operator}
\label{sec:ANNwithef}
	We propose a modification to the representation of a neuron in a classical neural network, by replacing the vector product of the input and weight with the $l1$ product defined in ef-operation. This modification can be applied to a wide range of artificial neural networks, including multi-layer perceptrons (MLP), recurrent neural networks (RNN) and convolutional neural networks (CNN).  

	A neuron in a classical neural network is represented by the following activation function;

	\begin{equation}
		\label{eq:classic_layer}
		f(\mathbf{x}\mathbf{W} + \mathbf{b}),
	\end{equation}
    
	where $\mathbf{W}\in\mathds{R}^{d\times M}$, $b_\in\mathds{R}^M$ are weights and biases, respectively, and $\mathbf{x}\in\mathds{R}^d$ is the input vector. 

	A neuron in the proposed additive neural network is represented by the activation function, where we modify the affine transform by using the ef-operator, as follows;
    
	\begin{equation}
		\label{eq:proposed_layer}
		f(\mathbf{a}\odot (\mathbf{x}\diamond \mathbf{W}) + \mathbf{b}),
	\end{equation}

	where $\odot$ is element-wise multiplication operator, $\mathbf{W}\in\mathds{R}^{d\times M}$, $a, b\in\mathds{R}^M$ are weights, scaling coefficients and biases, respectively, and $\mathbf{x}\in\mathds{R}^d$ is the input vector. The neural network, where each neuron is represented by the activation function defined in Eq.~\ref{eq:proposed_layer}, is called additive neural network.

	Comparison of Eq.~\ref{eq:classic_layer} and Eq.~\ref{eq:proposed_layer} shows that the proposed additive neural networks are obtained by simply replacing the affine scoring function ($\textbf{xW+b}$) of a classical neural network by the scoring function function defined over the ef-operator, $(\mathbf{a}\odot (\mathbf{x}\diamond \mathbf{W}) + \mathbf{b})$. Therefore, most of the neural networks can easily be  converted into the additive network by just representing the neurons with the activation functions defined over ef-operator, without modification of the topology and the general structure of the optimization algorithms of the network. 

\subsection{Training the Additive Neural Network}

	Standard back-propagation algorithm is applicable to the proposed additive neural network with small approximations. Back-propagation algorithm computes derivatives with respect to current values of parameters of a differentiable function to update its parameters. Derivatives are computed iteratively using previously computed derivatives from upper layers due to chain rule. Activation function, $f$, can be excluded during these computations for simplicity as its derivation depends on the specific activation function and choice of activation function does not affect the remaining computations. Hence, the only difference in the additive neural network training is the computation of the derivatives of the argument, $(\mathbf{a}\odot (\mathbf{x}\diamond \mathbf{W}) + \mathbf{b})$, of the activation function with respect to the parameters, $\mathbf{W},\mathbf{a},\mathbf{b}$, and input, $\mathbf{x}$, as given below:


	\begin{equation}
		\label{eq:dhda}
		\frac{\partial (\mathbf{a}\odot (\mathbf{x}\diamond \mathbf{W}) + \mathbf{b})}{\partial \mathbf{a}} = Diag(\mathbf{x}\diamond \mathbf{W}),
	\end{equation}

	\begin{equation}
		\label{eq:dhdb}
		\frac{(\partial \mathbf{a}\odot (\mathbf{x}\diamond \mathbf{W}) + \mathbf{b})}{\partial \mathbf{b}} = I_M,
	\end{equation}

	\begin{equation}
		\label{eq:dhdx}
		\begin{aligned}
			\frac{\partial (\mathbf{a}\odot (\mathbf{x}\diamond \mathbf{W}) + \mathbf{b})}{\partial \mathbf{x}_i} =& \begin{bmatrix}
					\mathbf{a}_1 (sign(\mathbf{W}_{i,1}) + 2\mathbf{W}_{i,1}\delta(\mathbf{x}_i))\\
					\vdots\\
					\mathbf{a}_M (sign(\mathbf{W}_{i,M}) + 2\mathbf{W}_{i,M}\delta(\mathbf{x}_i))
				\end{bmatrix}\\
			\approx& \, \mathbf{a} \odot \mathbf{sign}(\mathbf{w}_i),
		\end{aligned}
	\end{equation}

	\begin{equation}
		\label{eq:dhdW}
		\begin{aligned}
			\frac{\partial (\mathbf{a}\odot (\mathbf{x}\diamond \mathbf{W}) + \mathbf{b})}{\partial \mathbf{W}_{i,j}} =& (a_j(sign(\mathbf{x}_i)+2\mathbf{x}_i\delta(\mathbf{W}_{i,j})))\mathbf{e}_j \\
			\approx& \, a_j\mathbf{x}_i\mathbf{e}_j, 
		\end{aligned}
	\end{equation}

	where $\mathbf{a},\mathbf{b}\in \mathds{R}^M$, and $\mathbf{W}\in\mathds{R}^{d\times M}$ are the parameters of the hidden layer, $\mathbf{x}\in\mathds{R}^d$ is the input of the hidden layer, $\mathbf{e}_i\in\mathds{R}^M$ is the $i$th element of standard basis of $\mathds{R}^M$, $\mathbf{w}_i$ is the $i$th column of $\mathbf{W}$, $\mathbf{sign}(\mathbf{w}_i) = \sum_{j=1}^Msign(\mathbf{W}_{i,j})\mathbf{e}_j$ for $i=1,\,\hdots,\,M$, $\delta$ is the dirac delta function. 

	The above derivatives can be easily calculated using the following equation suggested by \cite{bracewellronaldnewbold2000}:

	\begin{equation}
		\frac{d}{dx} sign(x) = 2\delta(x) .
	\end{equation}
    
Approximations to derive the above equation are based on the fact that $\delta(x)=0$, almost surely.

\subsection{Existence and Convergence of the Solution in Additive Neural Network}

	In this section, first, we show that the proposed additive neural network satisfies the universal approximation property of \cite{cybenko1989approximation}, over the space of Lebesgue integrable functions. In other words. there exists solutions computed by the proposed additive network, which is equivalent to the solutions obtained by activation function with classical vector product. Then, we make a brief analysis for the convergence properties of the back propagation algorithm when the vector product is replaced by the ef-operators.
    
\subsubsection{Universal Approximation Property}

	The universal approximation property of the suggested additive neural network is to be proved for each specific form of the activation function. In the following proposition, we suffice to provide the proofs of universal approximation theorem for linear and ReLU activation functions, only. The proof (if it exits) for a general activation function requires a substantial amount of effort, thus it is left to a future work.

	\begin{proposition}
		\label{prop:approximation}
		The additive neural network, defined by the neural activation function with identity
        
		\begin{equation}
			\label{eq:12}
			f(\mathbf{a}\odot (\mathbf{x}\diamond \mathbf{W}) + \mathbf{b})=\mathbf{a}\odot (\mathbf{x}\diamond \mathbf{W}) + \mathbf{b},
		\end{equation}
        
		or an activation function with Rectified Linear Unit,
        
		\begin{equation}
			f(\mathbf{a}\odot (\mathbf{x}\diamond \mathbf{W}) + \mathbf{b})= ReLU(\mathbf{a}\odot (\mathbf{x}\diamond \mathbf{W}) + \mathbf{b}),
		\end{equation}

		is dense in $L^1(I_n)$. 
	\end{proposition}

	In order to prove the above proposition, the following two lemmas are proved first:

	\begin{lemma}
		\label{lemma:sign}
		If activation function $f$ is taken as identity (as in Eq.~\ref{eq:12}), then there exist additive neural networks, defined over the ef-operator, which can compute $f(x) = sign(\mathbf{y}^T\mathbf{x}+b)$, for any $\mathbf{y}\in\mathds{R}^d$ and $b\in\mathds{R}$.  
	\end{lemma}
	
    \begin{proof}
		Constructing an additive neural network, defined over ef-operator, is enough to prove the lemma. We can construct explicitly a sample network for any given $\mathbf{y}\in\mathds{R}^d$ and $b\in\mathds{R}$.  One such network consists of four hidden layers for $d=2$, this network can easily extended into higher dimensions. Let $\mathbf{x}$ be $[x_0, x_1]^T$ and $\mathbf{y}$ be $[y_0, y_1]^T$, then four hidden layers with following parameters can compute $f(x)=sign(\mathbf{y}^T\mathbf{x}+b)$.

        \begin{itemize}
            \item Hidden layer 1,

                $\mathbf{a}_1 = [y_1,\,y_1,\,y_1,\,y_2,\,y_2,\,y_2]^T$,

                $\mathbf{b}_1 = [0,\,0,\,0,\,0,\,0,\,0]^T$,

                $\mathbf{W}_1 = \begin{bmatrix}
                    1 & 1 & 2 & 0 & 0 & 0\\
                    0 & 0 & 0 & 1 & 1 & 2
                \end{bmatrix}$.

            \item Hidden layer 2,

                $\mathbf{a}_2 = [1]$,

                $\mathbf{b}_2 = [b]$,

                $\mathbf{W}_2 = 
                    \begin{bmatrix}
                        1 & 1 & -2 & 1 & 1 & -2
                    \end{bmatrix}^T$.
            \item Hidden layer 3,

                $\mathbf{a}_3 = \begin{bmatrix}1 & 1\end{bmatrix}^T$,

                $\mathbf{b}_3 = \begin{bmatrix}0 & 0\end{bmatrix}^T$,

                $\mathbf{W}_3 = 
                    \begin{bmatrix}
                        2 & 1
                    \end{bmatrix}$.
            \item Hidden layer 4,

                $\mathbf{a}_4 = [1]$,

                $\mathbf{b}_4 = [0]$,

                $\mathbf{W}_4 = 
                    \begin{bmatrix}
                        1 \\ -1
                    \end{bmatrix}^T$.
        \end{itemize}

        The function computed by this network can be simplified using the fact that, $\forall a,u\in \mathds{R}$ and $\forall b\in \mathds{R}^+$, 

        \begin{equation}
            \label{eq:sign_fact}
            sign(a(u+bsign(u)))=sign(au).
        \end{equation}

        Then, the hidden layers $h_1, h_2, h_3$ and $h_4$ can be represented as follows;

        \begin{equation}
            \label{eq:network_output}
            \begin{aligned}
                \mathbf{h}_1=\mathbf{a}_1 \odot (\mathbf{x} \diamond \mathbf{W}_1) + \mathbf{b}_1 &=
                \begin{bmatrix}
                    y_1(x_1 + sign(x_1))\\
                    y_1(x_1 + sign(x_1))\\
                    y_1(x_1 + 2sign(x_1))\\
                    y_2(x_2 + sign(x_2))\\
                    y_2(x_2 + sign(x_2))\\
                    y_2(x_2 + 2sign(x_2))
                \end{bmatrix}\\
                \mathbf{h}_2=\mathbf{a_2}\odot(\mathbf{h}_1\diamond \mathbf{W}_2)+\mathbf{b}_2 &= \mathbf{y}^T\mathbf{x}+b\\
                \mathbf{h}_3=\mathbf{a_3}\odot(\mathbf{h}_2\diamond \mathbf{W}_3)+\mathbf{b}_3 &= \begin{bmatrix} \mathbf{h}_2+2sign(\mathbf{h}_2)\\\mathbf{h}_2+sign(\mathbf{h}_2)\end{bmatrix}\\  
                \mathbf{h}_4=\mathbf{a_3}\odot(\mathbf{h}_3\diamond \mathbf{W}_4)+\mathbf{b}_4 &= sign(\mathbf{y}^T\mathbf{x}+b)
            \end{aligned}
        \end{equation}
    \end{proof}

	\begin{lemma}
		\label{lemma:relu_equiv}
		If the function $g(x)$ can be computable with activation function 

		\begin{equation}
			f(\mathbf{a}\odot (\mathbf{x}\diamond \mathbf{W}) + \mathbf{b})=\mathbf{a}\odot (\mathbf{x}\diamond \mathbf{W}) + \mathbf{b},
		\end{equation}

		then there exist an additive neural network architectures with a Rectified Linear Unit activation function,  

		\begin{equation}
			f(\mathbf{a}\odot (\mathbf{x}\diamond \mathbf{W}) + \mathbf{b})= ReLU(\mathbf{a}\odot (\mathbf{x}\diamond \mathbf{W}) + \mathbf{b}),
		\end{equation} 

		which can also compute $g(x)$. 
	\end{lemma}
    
	\begin{proof}
		This lemma can be proven using the following simple observations,
		
        \begin{itemize}
			\item Observation 1: If 
    			
                \begin{equation}
    				g(\mathbf{x}) = \mathbf{a}\odot(\mathbf{x}\diamond\mathbf{w})+\mathbf{b},
    			\end{equation} 
    			
                then, 
    
    			\begin{equation}
    				-g(\mathbf{x})= \mathbf{a}'\odot(\mathbf{x}\diamond\mathbf{w}')+\mathbf{b}',
    			\end{equation}
                
    			where $\mathbf{a}'=\mathbf{a}$, $\mathbf{w}'=-\mathbf{w}$, and $\mathbf{b}'=-\mathbf{b}$.
            
            \item  Observation 2: If 
    			
                \begin{equation}
    				g(\mathbf{x}) = \mathbf{a}\odot(\mathbf{x}\diamond\mathbf{w})+\mathbf{b},
    			\end{equation}
                
    			then,
    
    			\begin{equation}
    				g(\mathbf{x}) = \mathbf{a}''\odot((-\mathbf{x})\diamond\mathbf{w}'')+\mathbf{b}'',
    			\end{equation}
                
    			where $\mathbf{a}'=\mathbf{a}$, $\mathbf{w}'=-\mathbf{w}$, and $\mathbf{b}'=\mathbf{b}$.
    
    		\item Observation 3: If 
    
    			\begin{equation}
    				g(\mathbf{x}) = \mathbf{a}\odot(\mathbf{x}\diamond\mathbf{w})+\mathbf{b},
    			\end{equation}
    			
                then,
    			
                \begin{equation}
    				g(\mathbf{x}) = \mathbf{a}'''\odot(ReLU(\mathbf{x})\diamond\mathbf{w}+ReLU(-\mathbf{x})\diamond\mathbf{w}''')+\mathbf{b}''',
    			\end{equation}
                
    			where $\mathbf{a}'''=\mathbf{a}$, $\mathbf{w}'''=-\mathbf{w}$, and $\mathbf{b}'''=\mathbf{b}$.
		\end{itemize}	

		Lets assume that there exists an additive neural network, defined over the ef-operator, using identity as activation function which can compute the function $g(x)$. We can extend each layer using Observation 1, to compute both $g(x)$ and $-g(x)$. Afterwards, we can replace zeros on the weights introduced during previous extension on each layer using Observation 3, to replace the activation function with ReLU. This works, because either $ReLU(x)$ or $ReLU(-x)$ is 0. The modified network is an additive neural network with ReLU activation function, which can compute the function $g(x)$.
	\end{proof}

\begin{figure}
	\centering
	\includegraphics[width=0.50\textwidth]{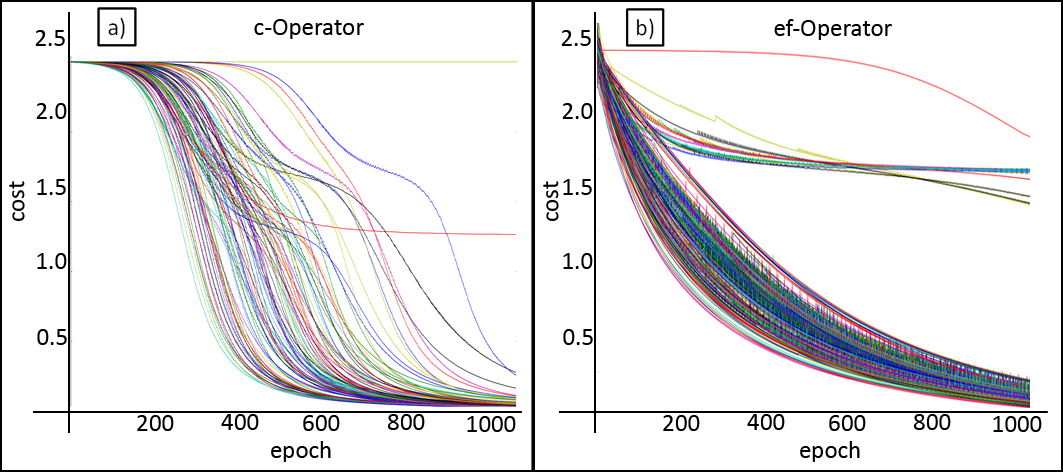}

	\caption{The plots of loss changes in the stochastic gradient descent (SGD) algorithm in the training phase of XOR problem while using single hidden layer MLP. While the Figure (1.a) shows the the changes of loss in the network by using classical score function \textbf{(c-operator)}, Figure (1.b) shows the loss changes in the same network with our proposed \textbf{(ef-operator)}. The results have been obtained by training the network 200 times in 1000 epochs which are shown by different colors.}
	\label{costFig} 
\end{figure}

\begin{proof}[Proof of Proposition \ref{prop:approximation}]
	This can be shown by the universal approximation theorem for bounded measurable sigmoidal functions \cite{cybenko1989approximation}. This theorem states that finite sums of the form

	\begin{equation}
		\label{eq:superposition_of_signs}
			G(\mathbf{x}; \{\alpha_i\}_{i=1}^N,\{\mathbf{y}_i\}_{i=1}^N,\{\theta_i\}_{i=1}^N) = \sum_{i=1}^N\alpha_i \sigma(\mathbf{y}_i^T\mathbf{x}+\theta_i),
	\end{equation}
	
	are dense in $L^1(I_n)$, where $\alpha_i, \theta_i \in \mathds{R}$ and $\mathbf{x},\mathbf{y}_i\in\mathds{R}^d$ for $i=1,\,2,\,\hdots,\,N$. It can be easily shown that $sign$ function is a bounded sigmoidal function. \textbf{Lemma} \ref{lemma:sign} shows that, if the activation function is taken as identity, then there exist networks which compute $sign(\mathbf{y}_i^T \mathbf{x} + \theta_i)$ for $i=1,\,2\,\hdots,\,N$. \textbf{Lemma} \ref{lemma:relu_equiv} shows that there are equivalent networks using ReLU as the activation function which compute the same functions. These networks can be combined with concatenation of layers of the additive neural networks to a single network. Also, proposed architecture contains fully connected linear layer at the output, and this layer can compute superposition of the computed $sign$ functions yielding $G(x)$. Since $G(\mathbf{x})$ can be computable by the additive neural networks, and $G(\mathbf{x})$ functions are dense in $L^1(I_n)$, then functions computed by the additive neural networks are also dense in $L^1(I_n)$.
\end{proof}

\subsubsection{Computational efficiency}

	The proposed additive neural network contains more parameters then the classical neuron representation in MLP architectures. However, each hidden layer can be computed using considerably less number of multiplication operator. A classical neural network, represented by the activation function $f(\mathbf{x}\mathbf{W}+\mathbf{b})$, containing $M$ neurons with $d$ dimensional input, requires $d\times M$ many multiplication operator to compute $\mathbf{x}\mathbf{W}+\mathbf{b}$. On the other hand, the additive neural network, represented by the activation function, $f(\mathbf{a}\odot(\mathbf{x}\diamond\mathbf{W})+\mathbf{b})$ with the same number of neurons and input space requires $M$ many multiplication operator to compute $\mathbf{a}\odot(\mathbf{x}\diamond\mathbf{W})+\mathbf{b}$. This reduction on number of multiplications is especially important when input size is large or hidden layer contains large number of neurons. If activation function is taken as either identity or ReLU, then output of this layer can be computed without any complex operations, and efficiency of the network can be substantially increased. Multiplications can be removed entirely, if scaling coefficients, $\mathbf{a}$ are taken as 1. However, these networks may not represent some functions, and consequently may perform poorly on some datasets.

\subsubsection{Optimization problems}

	Due to the sign operation performed in each neuron, the  ef-operator creates a bunch of hyperoctants in the cost function at each layer of the additive neural network. Therefore, the local minima computed at each layer, depends on the specific hyperoctant for a set of weights. The change in the signs results in a jump from a hyperoctant to another one. 
    
	For some datasets, some of the local minima may lie on the boundaries of the hyperoctants. Since the hyperoctants are open sets, this may leave some hyperoctands with non-existing local minima. A gradient based search algorithm may update the weights such that the algorithm converges to the  local minima on the boundary. If the step size and number of epochs are increased, then  the updated weights leave the current hyperoctant without converging to a local minima on the boundary and new set of weights make the algorithm to converge to a local minima in another hyperoctant. However, the new hyperoctant may have the same problem.


\section{Experimental Results}
 

\begin{table*}[h]
	\centering
	\caption{Optimal classification results of classic function c-operator and our proposed ef-operator score functions in different MLP (in different learning rates) and LeNet-5 architectures.}
	\begin{tabular}{|l|c||c|c||c|c||c|c|}
		\hline
		\textit{\textbf{Architectures}}                    
		\label{resulttable}
		& \textbf{-}   &\multicolumn{2}{c||}{\textbf{ReLU}}  &\multicolumn{2}{c||}{\textbf{Tanh}} &\multicolumn{2}{c|}{\textbf{Sigmoid}}     \\ \hline 

		&\small{learning rate}  &\small{c-operator}  &\small{ef-operator} &\small{c-operator}  &\small{ef-operator}    &\small{c-operator}    &\small{ef-operator}   \\ \hline \hline

		\multirow{4}{*}{MLP (\small{2 Hidden Layers}) }  &0.01     & \textbf{98.43}  & 98.01  &96.39  &95.57  &97.81   &96.80 \\ 
                              &0.005    & 98.36  & \textbf{98.09}  &97.23  &96.05  &98.07   &97.10 \\ 
                              &0.001    & 98.03  & 97.76  &97.63  &96.77  &95.83   &96.47\\ 
                              &0.0005   & 97.61  & 97.21  &96.27  &96.10  &95.83   &95.53 \\ \hline 

		\multirow{4}{*}{MLP (\small{3 Hidden Layers})} &0.01    &96.85   &97.80 &90.42 &92.64 &96.31  &96.23 \\ 
                              &0.005   &98.15   &\textbf{97.95} &95.08 &93.33 &96.48  &96.50 \\ 
                              &0.001   &\textbf{98.22}   &97.63 &97.49 &93.63 &95.74  &95.85\\ 
                              &0.0005  &97.65   &96.97 &96.78 &93.93 &94.34  &94.83 \\ \hline \hline

\multirow{3}{*}{LeNet-5}  &  &    &   &   &   &    &  \\
						  &\small{-} & \textbf{99.29}  & \textbf{98.60}  &99.22  &98.43  &99.20   &97.81  \\
                          &  &    &   &   &   &    &  \\
\hline
	\end{tabular}
\end{table*}
\begin{figure*}
	\centering
	\includegraphics[width=16.6cm,height=4cm,keepaspectratio]{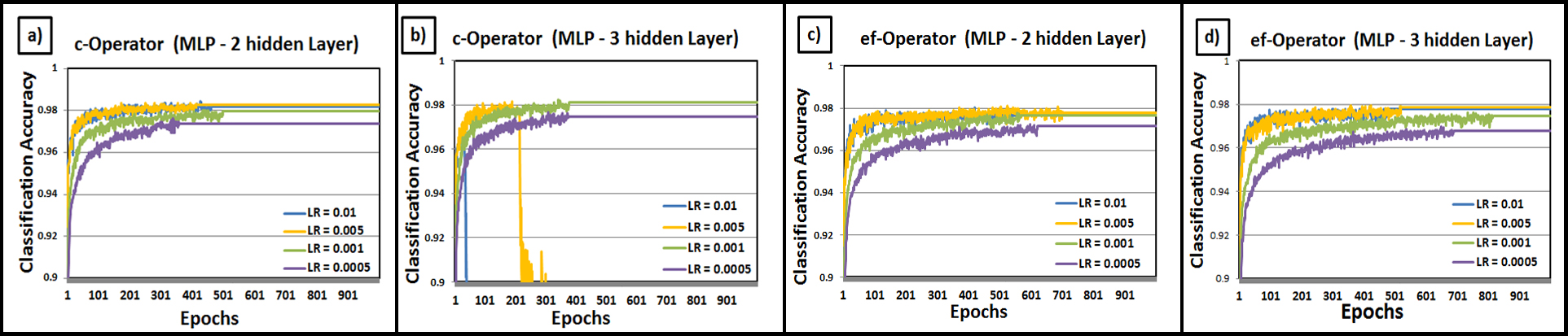}
	\caption{Plots of classification accuracies in different architectures with different score functions. Subplots (a) and (b) shows the results of MLP with 2 and 3 hidden layers using classic c-operator. Subplots (c) and (d) shows the results of MLP with 2 and 3 hidden layers using our proposed ef-operator.}
	\label{ResutlsCurve} 
\end{figure*}


	Multi-layer perceptron (MLP) \cite{Bishop} is used to measure the ability of the proposed additive neural network, in machine learning problems.  MLP consists of a single input and output layer and multiple hidden layers. The size and the number of hidden layers can vary a great deal, depending on the problem domain. In this research, we use one, two and three hidden layers, respectively, in two different classification problems, namely XOR problem and character recognition of MNIST dataset. The input layer receives pattern sample $x \in  {R}^\textbf{D}$ to the network. 

	On the other hand, the hidden layer(s) contains biological inspired units called neurons which learns a new representations from the input patterns. Each neuron consists of a scoring function and an activation function. As discussed in the Section~\ref{sec:ANNwithef}, the scoring function is an affine transform in the form of $(\mathbf{x}\mathbf{W} + \mathbf{b})$ in the classic neural network where $\mathbf{x}$ and $\mathbf{b}$ are the parameters. In this study, we call the widely used classic scoring function  $(\mathbf{x}\mathbf{W} + \mathbf{b})$ as c-operator. As discussed in the Section~\ref{sec:ANEEO} and \ref{sec:ANNwithef}, the proposed score function, ef-oprerator, is an energy efficient alternative of the classical vector product.  

	In addition to the score function, each neuron of a hidden layer also has an activation function that makes the network nonlinear. Several activation functions such as sigmoid, hyperbolic tangent (Tanh) and rectified linear unit (ReLU) functions have been used as the activation function. While some studies such as \cite{krizhevsky2012imagenet} have shown that ReLU outperform the others in most of the cases, we also examined sigmoid and Tanh in the following experiments. Finally, the last layer of MLP, called output layer, maps the final hidden layer to the scores of the classes by using its own score function. We used both the classical c-operator and the new ef-operator at the output layer to make the final decision. 
 
	The aim of MLP is to find the optimal values for parameters $\mathbf{W}$ and $\mathbf{b}$ using backpropagation \cite{BackProp} and optimization algorithms such as stochastic gradient descent (SGD). In order to implement the network, Tensorflow \cite{tensorflow2015}, a python library for numeric computation, is used.

	In the first experiment, we examine the ability of additive neural network to partition a simple nonlinear space, solving the XOR problem. We compare the classical MLP with affine scoring function and additive neural network with ef-operator. Since a single hidden layer MLP with c-operator can solve XOR problem, we used one hidden layer in both classical and the proposed architectures. Mean squared error is used as cost function to measure the amount of loss in training phase of the network, and we fixed the number of neurons in the hidden layer to 10.
    
	The additive neural network with ef-operator could successfully solve the XOR problem and reached to $100\%$ accuracy in this problem. We also investigate the rate of changes inloss changes at each epoch. It is also notable that some of the runs that are shown by colors, do not reach to minimum values in 1000 epochs. This shows that more epochs is needed in some runs. Generally, the number of epochs depends on learning rate and initialization condition, and the final epoch can be determined by some stopping criteria. However, in this study, we are only interested to see the variations in the cost; therefore, we fixed the number of epochs to 1000.

	Left and right sides of Fig.~\ref{costFig} show the change of loss in the MLP using c-operator and ef-operator, respectively, with ReLU as the activation function. We rerun the network for 200 times in 1000 epochs, and used k-fold cross validation to specify the learning-rate parameter of SGD. Each color of the plots shows the variations in loss or cost value (x axis) across the epochs (y axis) in one specific run of the network. As the figure shows, the cost value of the network with our proposed ef-operator decreases along the epochs and acts similar to classical affine operator, called c-operator.        

	In the second experiment, we classified the digits of MNIST dataset of \cite{lecun1998gradient} which consists of handwritten examples to examine our proposed additive neural network in multiclass classification problem. MNIST dataset consists of 30,000 training samples and 5,000 test data. Each example is an image of a digit from 0 to 9. One-hot code is used to encode the class labels. Each example is an image of size $28 \times 28$, and each image is concatenated in a single vector to input the network. Therefore, the size of the input layer of the network is 784. We used cross-entropy based cost function and SGD to train the network. We used 150 number of examples in each iteration of SGD. In other words, the batch size is equal to 150.

	Table~\ref{resulttable} contains the classification accuracies of the MLP architecture using three activation functions: ReLU, Tanh and Sigmoid with four different learning rates. As the table shows, our additive neural network over ef-operator reaches to the performance of classic MLP with c-operator. In other words, with a slightly sacrificing the classification performance we can use the proposed ef-operator which much more energy-efficient. Note that, we have not used any regularization methods such as drop out used by Krizhevsky et al. \cite{krizhevsky2012imagenet}, because we simply aim to show that our proposed ef-operator gives the learning ability to the deep MLP. Also Table.~\ref{resulttable} shows that maximum of the performances have been obtained using ReLU activation function. We are also interested to see the variations in the classification performances during the epochs and along the epochs.



With addition to MLP, we have used the proposed ef-operator to learn the parameters of LeNet-5 \cite{lecun1998gradient} to classifying MNIST dataset. Table~\ref{resulttable} contains the classification accuracy of LeNet-5 architecture that contains two conventional and one fully connected layer.  We trained the network with SGD and cross-entropy based cost functions as we did on MLP case. It should be noted that we have used the conventional c-operator in the output layer of both MLP and LeNet-5 architectures. As shown in the table, the proposed ef-operator catches up the c-operator with a small amount of loss. 

	Figure~\ref{ResutlsCurve} shows the results of the classification accuracies obtained from MLP based on our proposed ef-operator and traditionally used c-operator. The performances (shown in the y axis of the sub figures) obtained in successive epochs (shown in the x axis of the sub figures). In each epoch, the network is trained with all of the training examples. The plots of the sub-figures are obtained using four different learning rates: 0.1, 0.005, 0.001 and 0.0005. Subplots (a) and (b) at the left of figure shows the results of c-operator in MLP with 2 and 3 hidden layers respectively, and subplots (c) and (d) shows the results of our proposed ef-operator. As Figure~\ref{ResutlsCurve} shows, our operator effectively increases the classification performance as the number of epochs increases and reaches nearly to the original linear function.


\section{Conclusion}
	In this study, we propose an energy efficient additive neural network architecture. The core of this architecture is the lasso norm based ef-operator that eliminates the energy-consumption multiplications in the conventional architecture. We have examined the universal approximation property of the proposed architecture over the space of Lebesgue integrable functions and test it in real world problems. We showed that ef-operator can successfully solve the nonlinear XOR problem. Moreover, we have observed that with sacrificing $0.39\%$ and $0.69\%$ accuracy, our proposed network can be used in the multilayer perceptron (MLP) and conventional neural network respectively to classify MNIST dataset. As a future work, we plan to test the proposed architecture in the state-of-the-art deep neural networks.

\section*{Acknowledgment}
A. Enis Cetin's work was funded in part by a grant from Qualcomm.


\end{document}